\documentclass[letterpaper]{article} % DO NOT CHANGE THIS
\usepackage{aaai25}  % DO NOT CHANGE THIS
\usepackage{times}  % DO NOT CHANGE THIS
\usepackage{helvet}  % DO NOT CHANGE THIS
\usepackage{courier}  % DO NOT CHANGE THIS
\usepackage[hyphens]{url}  % DO NOT CHANGE THIS
\usepackage{graphicx} % DO NOT CHANGE THIS
\usepackage{natbib}  % DO NOT CHANGE THIS AND DO NOT ADD ANY OPTIONS TO IT
\usepackage{caption} % DO NOT CHANGE THIS AND DO NOT ADD ANY OPTIONS TO IT
\usepackage{algorithm}
\usepackage{algorithmic}
\usepackage{caption}
\usepackage{subcaption}
\usepackage{newfloat}
\usepackage{listings}
\DeclareCaptionStyle{ruled}{labelfont=normalfont,labelsep=colon,strut=off} % DO NOT CHANGE THIS
\floatstyle{ruled}
\newfloat{listing}{tb}{lst}{}
\floatname{listing}{Listing}
\usepackage{tcolorbox}

\usepackage{amsmath}
\usepackage{amsfonts}
\usepackage{amsthm}
\usepackage{amssymb}
\theoremstyle{plain}
\newtheorem{theorem}{Theorem}%[section]
\newtheorem{proposition}[theorem]{Proposition}

\theoremstyle{definition}

\theoremstyle{remark}
\newtheorem{remark}[theorem]{Remark}

\newtheorem*{theorem*}{Theorem}
\newtheorem*{lemma*}{Lemma}
\newtheorem*{definition*}{Definition}
\newtheorem*{corollary*}{Corollary}
\newtheorem*{remark*}{Remark}

\DeclareMathOperator*{\E}{\mathbb{E}}
\DeclareMathOperator*{\argmax}{argmax}   % Jan Hlavacek

\newcommand{\Es}{\E_{s'\sim{} p}}
\newcommand{\s}{\mathcal{S}}
\newcommand{\R}{\mathbb{R}}
\newcommand{\A}{\mathcal{A}}

\newcommand{\T}{\mathcal{T}}

\usepackage{url}
\title{Bootstrapped Reward Shaping}
\author{
    Jacob Adamczyk\textsuperscript{\rm 1,2}, Volodymyr Makarenko\textsuperscript{\rm 3}, Stas Tiomkin\textsuperscript{\rm 4}, Rahul V. Kulkarni\textsuperscript{\rm 1,2}
}
\affiliations{
    \textsuperscript{\rm 1}Department of Physics, University of Massachusetts Boston\\
    \textsuperscript{\rm 2}The NSF AI Institute for Artificial Intelligence and Fundamental Interactions\\
    \textsuperscript{\rm 3}Department of Computer Engineering, San Jos\'e State University\\
    \textsuperscript{\rm 4}Department of Computer Science, Whitacre College of Engineering, Texas Tech University
    jacob.adamczyk001@umb.edu, volodymyr.makarenko@sjsu.edu, stas.tiomkin@ttu.edu, rahul.kulkarni@umb.edu
}

\usepackage{bibentry}
\begin{document}

\maketitle

\begin{abstract}
    In reinforcement learning, especially in sparse-reward domains, many environment steps are required to observe reward information. In order to increase the frequency of such observations, ``potential-based reward shaping'' (PBRS) has been proposed as a method of providing a more dense reward signal while leaving the optimal policy invariant. However, the required ``potential function'' must be carefully designed with task-dependent knowledge to not deter training performance. In this work, we propose a ``bootstrapped'' method of reward shaping, termed BSRS, in which the agent's current estimate of the state-value function acts as the potential function for PBRS. We provide convergence proofs for the tabular setting, give insights into training dynamics for deep RL, and show that the proposed method improves training speed in the Atari suite.% discrete and continuous control environments.
\end{abstract}
% The theory of BSRS is motivated by recent work in transfer learning and is related to ``Munchausen-DQN''.
% \textbf{Due to its simple nature, it has the opportunity to be added on to any other value based learning algorithm}

\section{Introduction}

The field of reinforcement learning has continued to enjoy successes in solving a variety of problems in both simulation and the physical world. However, the practical use of reinforcement learning in large-scale real-world problems is hindered by the enormous number of environment interactions needed for convergence. Furthermore, even defining the reward functions for such problems (``reward engineering'') has proven to be a significant challenge. Improper design of the reward function can inadvertently change the optimal policy, leading to suboptimal or undesirable behaviors, while attempts to create more dense or interpretable reward signals often come at the cost of task complexity. Historically, earlier attempts to adjust the reward function through ``reward shaping'' indeed resulted in unpredictable and negative changes to the corresponding optimal policy~\cite{bicycle}. 

A significant breakthrough in the field of reward shaping came with the introduction of  Potential-Based Reward Shaping (PBRS) from~\cite{ng_shaping}. PBRS provided a theoretically-grounded method for changing the reward function while keeping the optimal policy fixed. This guarantee framed PBRS as an attractive method of injecting prior knowledge or domain expertise into the reward function through the use of the so-called ``potential function''. PBRS was shown to greatly improve the training efficiency in grid-world tasks with suitably-chosen potential functions (e.g. the Manhattan distance to the goal). However, this potential function is limited to simple goal-reaching MDPs and requires external task-specific knowledge to handcraft. In more complex domains, especially where task-specific knowledge is not already present, a globally applicable method for choosing a potential function would be advantageous.
% Foundational work in the field of PBRS \cite{ng_shaping} showed that for some potential functions (e.g. distance-based heuristics in goal-reaching tasks), PBRS is successful at reducing the sample complexity required to find the optimal policy.
% reconcile
% Rather than arbitrarily changing the reward function to encourage faster learning, a specific method of changing the reward function has been derived which is proven to not change the optimal policy~\cite{ng_shaping}. This method, Potential-Based Reward Shaping (PBRS) enables changes to the reward function defined by a real-valued potential function, typically denoted $\Phi(s)$. 

\citeauthor{ng_shaping} suggested that using the \textit{optimal} state-value function for the choice of potential: ${\Phi(s)\doteq V^*(s)}$ may be useful, as it encodes the optimal values of states, requiring only the remaining non-optimal $Q$-values to be learned. This idea was further studied in later work such as \cite{zou2021learning}. However, this approach presents a circular problem: it requires knowledge of the optimal solution to aid in finding the optimal solution, making it impractical for single-task learning scenarios.%, this shaping potential presupposes the knowledge of the \textit{solution}, $V^*$, a difficult assumption to surmount in the single-task setting, without first solving the task of interest. 

In this work, we propose a novel approach to choosing the potential function termed BootStrapped Reward Shaping (BSRS). Rather than using the (unknown) optimal value function $V^*$ itself for the potential function, we use the next most reasonable choice for a state-dependent function readily available to the agent: the \textit{current estimate of the optimal state-value function}, $V^{(n)}$. This approach leverages the agent's continually improving estimate of the optimal state value function ($V^*$) at step $n$ to introduce a more dense reward signal.
By bootstrapping from the agent's best current estimate of $V^*$, we introduce an adaptive reward signal that evolves with the agent's understanding of the task, while remaining tractable.

The use of a time-dependent potential function raises important questions of convergence, which we address theoretically and empirically in the following sections. Next, we provide experiments in both tabular and continuous domains with the use of deep neural networks. We find that the use of this simple but dynamical potential function can improve sample complexity, even in complex image-based Atari tasks \cite{ALE}. Moreover, the proposed algorithm requires \textbf{changing only a single line of code} in existing value-based algorithms.%, making it an attractive method for implementation in the deep RL community.

The broader implications of the present work extend beyond immediate performance improvements. By providing a general, adaptive approach to reward shaping, BSRS opens new avenues for tackling complex RL problems where effective reward design is challenging. Our work also contributes to the active discussion on (1) the role of reward shaping in RL and (2) methods for leveraging an agent's growing knowledge to accelerate learning. The present work also opens new directions for future research, which we discuss at the end of the paper. Our code is publicly available at~\url{https://github.com/JacobHA/ShapedRL}.

% PBRS proved by Ng, Wiewiora, AAAI gives a method of altering the reward funciton in a way that leaves the optimal policy invariant. However it requires knowing a ``potential'' function that is good. The best potential is the optimal value function V* whihc is of course not obtained until convergence. However, other methods have been propsoed for the potential function. In fact, any state function can be used.

% We propose to use as potential the target value function. We show that this method converges to twice the value of original unshaped problem. In practice, we find that this shaping method does yield sopme improvement in performance sample complexity. 

% Our AAAI paper

Overall, the main contributions of this work can be summarized as follows:
\begin{itemize}
    \item We propose a novel mechanism of dynamic reward shaping based only on the agent's present knowledge: BSRS.
    \item We prove the convergence of this new method.
    \item We show an experimental advantage in using BSRS in tabular grid-worlds, the Arcade Learning Environment, and locomotion tasks.
\end{itemize}

\section{Background}
In this section we will introduce the relevant background material for reinforcement learning and potential-based reward shaping (PBRS). 
\subsection{Reinforcement Learning}
We will consider discrete or continuous state spaces and discrete action spaces (this restriction is for an exact calculation of the state-value function; in the final sections we discuss extensions for continuous action spaces). The RL problem is then modeled by a Markov Decision Process (MDP), which we represent by the tuple $\langle \s,\A,p,r,\gamma \rangle$ with state space $\s$; action space $\A$; potentially stochastic transition function (dynamics) $p: \s \times \A \to \s$; bounded, real reward function $r: \s \times \A \to \mathbb{R}$; and the discount factor $\gamma \in [0,1)$.

The defining objective of reinforcement learning (RL) is to maximize the total discounted reward expected under a policy $\pi$. That is, to find a policy $\pi^*$ which maximizes the following sum of expected rewards:
\begin{equation}\label{eq:pi_star_defn}
    \pi^* = \arg\max_{\pi}
    \E_{\tau \sim{}p,\pi}
    %_{\rho_\pi}%_{\pi(a|s),p(s'|s,a)}
    \left[ \sum_{t=0}^{\infty} \gamma^{t} r(s_t,a_t) %- \frac{1}{\beta}
        % \pi^0(a_t|s_t)\log\left(\frac{\pi(a_t|s_t)}{\pi_0(a_t|s_t)} \right) 
         \right].
\end{equation}

In the present work, we restrict our attention to value-based RL methods, where the solution to the RL problem is equivalently defined by its optimal action-value function ($Q^*(s,a)$). The aforementioned optimal policy $\pi^*(a|s)$ is derived from $Q^*$ through a greedy maximization over actions (Eq.~\eqref{eq:greedy-pi}). The optimal value function can be obtained by iterating the following recursive Bellman equation until convergence:
% In un-regularized RL, the Bellman optimality equation is given by \cite{suttonBook}:
\begin{equation}
    Q^*(s,a) = r(s,a) + \gamma \mathbb{E}_{s' \sim{} p(\cdot|s,a)} \max_{a'} \left( Q^*(s',a') \right).
    \label{eq:bellman}
\end{equation}
In the tabular setting, the exact Bellman equation shown above can be applied until convergence within some numerical tolerance. In the function approximation setting (e.g. with deep neural nets), the $Q$ table is replaced by a parameterized function approximator, denoted $Q_\theta$, and the temporal difference (TD) loss is minimized instead:
\begin{equation}
    \mathcal{L}_\theta = \sum_{s,a,r,s' \in \mathcal{D}}\biggr|Q_\theta(s,a) - \left(r + \gamma \max_{a'} Q_{\bar{\theta}}(s',a')\right)\biggr|^2
\end{equation}
via stochastic gradient descent on the ``online'' value network's parameters, ${\theta}$. The target value network $Q_{\bar{\theta}}$ is used to compute the next step's value. Its parameters ${\bar{\theta}}$ represent a lagging version of the online parameters, typically calculated through periodic freezing of the online network or more general Polyak averaging. During training, experience from the environment is collected by an $\varepsilon$-greedy exploration policy and stored in a FIFO replay buffer $\mathcal{D}$. Uniformly sampled mini-batches are sampled from $\mathcal{D}$ to compute the loss $\mathcal{L}_\theta$.

Once the optimal action-value function is obtained, the unique deterministic optimal policy may be immediately derived from it as follows:
\begin{equation}\label{eq:greedy-pi}
    \pi^*(a|s) = \argmax_a Q^*(s,a).
\end{equation}

In the following section, we discuss methods for altering the reward function $r(s,a)$ in a way that leaves the optimal policy $\pi^*$ unchanged.

% \sout{ \ST{Eq.1 is defined for any policy $\pi^0$ might arise questions}} \JA{Am I missing something? I am looking at S+B's eq 4.2}
% The entropy term in the objective function of entropy-regularized RL modifies the previous optimality equation in the following way
% \cite{ZiebartThesis, Haarnoja_SAC}:
% \begin{equation}
%     Q^*(s,a) = r(s,a) + \frac{\gamma}{\beta} \E_{s' \sim{} p
%     %p(s'|s,a)} 
%     }\log \E_{a' \sim{} \pi_0}
%     %\sim{} \pi^0(a'|s')} 
%     e^{ \beta Q^*(s',a') }.
%     \label{eq:soft_bellman}
% \end{equation}

% The regularization parameter $\beta$ is used to control the degree of stochasticity in the optimal policy. In the entropy-regularized setting, $Q^*$ is referred to as the optimal ``soft'' action-value function. For brevity, we will hereon refer to $Q^*$ simply as the value function when context is clear.

\subsection{Potential-Based Reward Shaping}

% In (un-regularized) RL, the agent's goal is to maximize the following objective function

% this is done in practice by using function approximators to estimate the value function $Q(s,a)$.
With the goal of efficiently learning an optimal policy for a given reward function, one may wonder how the reward function can be adjusted\footnote{The literature sometimes refers to ``reward shaping'' as arbitrary changes to the reward function. To avoid confusion, we will use ``shaping'' only in the context of PBRS.} to enhance training efficiency. Arbitrary ``reward engineering'' may improve performance \cite{hu2020learning} but is not guaranteed to yield the same optimal policy. Indeed, arbitrary changes to the reward function may result in the agent performing reward ``hacking'' or ``hijacking'', having detrimental effects on solving the originally posed task. Examples of this in prior work include the discussion on cycles in ~\cite{ng_shaping} and more involved examples were later given by \cite{zhang2021importance, skalse2022defining}. Instead of arbitrary changes, a specific additive\footnote{In the current setting of un-regularized $Q$-learning, positive scalar multiplication of the reward function also leaves the optimal policy invariant, but we focus on additive changes in this work.} method of changing the reward function, proved by \cite{ng_shaping} to leave the optimal policy invariant, is given by the following result: \textit{potential-based} reward shaping (PBRS).

\begin{theorem}[\citet{ng_shaping}]
Given task ${\T} = \langle \s,\A,p,{r},\gamma \rangle$ with optimal policy ${\pi}^*$, then the task $\widetilde{\T}$ %=\langle \s,\A,p,r',\gamma, \beta, \pi_0 \rangle$ 
with reward function
\begin{equation}\label{eq:eq for rewards in thm rwd shaping}
    \widetilde{r}(s,a) = {r}(s,a) + \gamma \Es \Phi(s') - \Phi(s) %\E_{s' \sim{} p}
\end{equation}
has the optimal policy $\widetilde{\pi}^* = {\pi}^*$, and its optimal value functions satisfy
\begin{equation}\label{eq:in thm rwd shaping q = q - phi}
    \widetilde{Q}^*(s,a) = {Q}^*(s,a) - \Phi(s)
\end{equation}
\begin{equation}\label{eq:in thm rwd shaping v = v - phi}
    \widetilde{V}^*(s) = {V}^*(s) - \Phi(s)
\end{equation}
for a bounded, but otherwise arbitrary function $\Phi \colon \s \to \mathbb{R}$.
\label{thm:ng}
\end{theorem}

% Intuitively, to keep the optimal policy fixed, there are an extra $|\s|$ degrees of freedom in the policy normalization condition at each state.
Intuitively, the form $\gamma \Phi(s') - \Phi(s)$ represents a discrete-time derivative along trajectories in the MDP~\cite{jenner2022calculus}. Including this shaping term in a trajectory's expected return (Eq.~\eqref{eq:pi_star_defn}), subsequent terms have a telescopic cancellation, leaving behind only $\Phi(s_0)$. %~\footnote{The potential's contribution at the end of the trajectory is zero, since $\gamma^N \Phi(s_N)\to 0$ as $N\to\infty$.}. 
This leads to the predictable effect on the value functions seen in Theorem~\ref{thm:ng}.

The results of Theorem~\ref{thm:ng} show that the RL problem designer has the freedom to choose \textit{any} $\Phi$ to shape the reward in a way that is consistent with the original MDP's solution. However, it importantly does not give any specific prescription for a ``preferred'' choice of $\Phi: \s\to\R$, which is up to the user to define. In the following section, we will discuss previously studied notions of preferred potentials, and provide extensions of the PBRS framework relevant to the present work.

\section{Prior Work}

The field of reward shaping in reinforcement learning has a rich history, with roots tracing back to early work on accelerating learning through reward design \cite{MATARIC1994181, bicycle}. However, the seminal work of \cite{ng_shaping} marked a significant turning point by introducing Potential-Based Reward Shaping (PBRS). This approach provided a theoretical foundation for modifying rewards without altering the optimal policy, a crucial property for maintaining the ``correct'' or ``desirable'' agent behavior at convergence.

The key insight of \cite{ng_shaping} was that shaping rewards based on a potential function as in Theorem~\ref{thm:ng} preserves the optimal policy. This result was further extended by \cite{wiewiora2003potential}, who proved the equivalence between PBRS and Q-value initialization, providing additional theoretical justification and understanding of the approach. Our work builds directly on these foundations, leveraging the policy invariance property of PBRS while introducing a novel, adaptive approach to defining the potential function.

Following the work of \citeauthor{ng_shaping}, several studies have explored extensions and applications of PBRS which we detail below. Firstly, Dynamic PBRS~\cite{devlin2012dynamic} extended PBRS to a dynamic setting where the potential function is time-dependent within the MDP\footnote{Here, ``time'' refers to the transition step in the MDP itself, rather than the training step in the algorithm.}. We note an important distinction from this work is the meaning of a ``dynamical'' potential function. In \cite{devlin2012dynamic}, the potential function changes at each discrete time-step in the MDP: $\Phi(s,t)$, and they prove convergence to the optimal policy despite a non-convergent potential function. However, their study mainly focuses on randomly generated potential functions and does not show improved performance. On the other hand, we study a potential function $\Phi(s)$ that is fixed across time-steps in the MDP, but varies at each training step. Thus, in our context, if many environment steps occur between gradient steps, the same potential $\Phi$ is used until the next update. We also prove convergence of our method and show empirically that the proposed shaping method (BSRS) can lead to faster training.

In the setting of entropy-regularized (``MaxEnt'') RL, \cite{centa2023soft, Adamczyk_AAAI} established connections between the prior policy, PBRS, compositionality, and soft Q-learning; broadening the theoretical understanding of reward shaping. Furthermore, their analysis shows that the degree of freedom used for shaping can be derived from the normalization condition on the optimal policy, or equivalently from an arbitrary ``base task''. Because of these results, our analysis readily extends to the more general entropy-regularized setting. For simplicity, this work will focus on the un-regularized case. 

PBRS has also assisted in furthering the theoretical understanding of bounds on the value function. For example, \cite{gupta2022unpacking, Adamczyk_UAI, Adamczyk_RLC} explored the relationship between PBRS and value function bounds, providing insights into the theoretical and experimental utility of shaping. PBRS has also been explored in the average-reward setting~\cite{jiang2021temporal, naik2024reward} where the latter work's ``dynamic'' but constant potential function can be connected to our dynamic and state-dependent potential function (where they use the mean policy reward instead of the associated value function to define the potential).

The field of inverse reinforcement learning (IRL), concerned with learning the underlying reward signal from expert demonstrations, has also benefited from the ideas of PBRS. As PBRS effectively describes the equivalence class of reward functions (with respect to optimal policies), IRL must take into account the potentially un-identifiable differences between seemingly different reward functions in the same equivalence class. In the IRL setting this has been studied thoroughly by e.g. \cite{cao2021identifiability}. Later, \cite{gleave2020quantifying, jenner2022general, wulfe2022dynamics} used this insight to develop a notion of reward distances, demonstrating the broader applicability of the PBRS framework for identifying the salient differences in potentially similar reward functions.
% Generalized Advantage Estimation: Schulman et al. [2015] showed connections between PBRS and generalized advantage estimation, linking reward shaping to policy gradient methods.

% Our work complements these extensions by introducing a novel, adaptive approach to defining the potential function, which can be integrated with many of these existing frameworks.

Despite the simplicity and applicability of PBRS, a persistent challenge has been the design of effective potential functions without heuristics. \citeauthor{ng_shaping} suggested using the optimal state-value function $V^*(s)$ as the potential, further explored by \cite{zou2021learning, cooke2023toward} in the meta-learning multi-task setting. However, in the single-task setting this approach presupposes knowledge of the solution, limiting its practical applicability.

Other approaches have included heuristic-based potentials \cite{ng_shaping},
learning-based potentials (especially in the hierarchical setting) \cite{grzes2010online, gao2015potential, ma2024highly}, and random dynamic potentials \cite{devlin2012dynamic}. Although these approaches have found utility in their respective problem settings, BSRS provides a universally applicable potential function which can be computed without requiring additional samples or training steps.
% \newline
% \newline
% Because of its simple nature, the proposed algorithm has a straightforward application to any existing value-based algorithms (e.g. DQN, SAC, TD3). 
% \section{Motivation}
% As stated in the previous sections, we follow the suggestion in \cite{ng_shaping} that $V^*$ is a ``good'' choice for the potential function. To see why, we can first see that $r+\gamma V^*(s')-V^*(s)=Q^*-V^*=$
\section{Theory\label{sec:theory}}
\begin{figure*}
\centering
\begin{minipage}[t]{0.8\textwidth}
  \includegraphics[width=\linewidth]{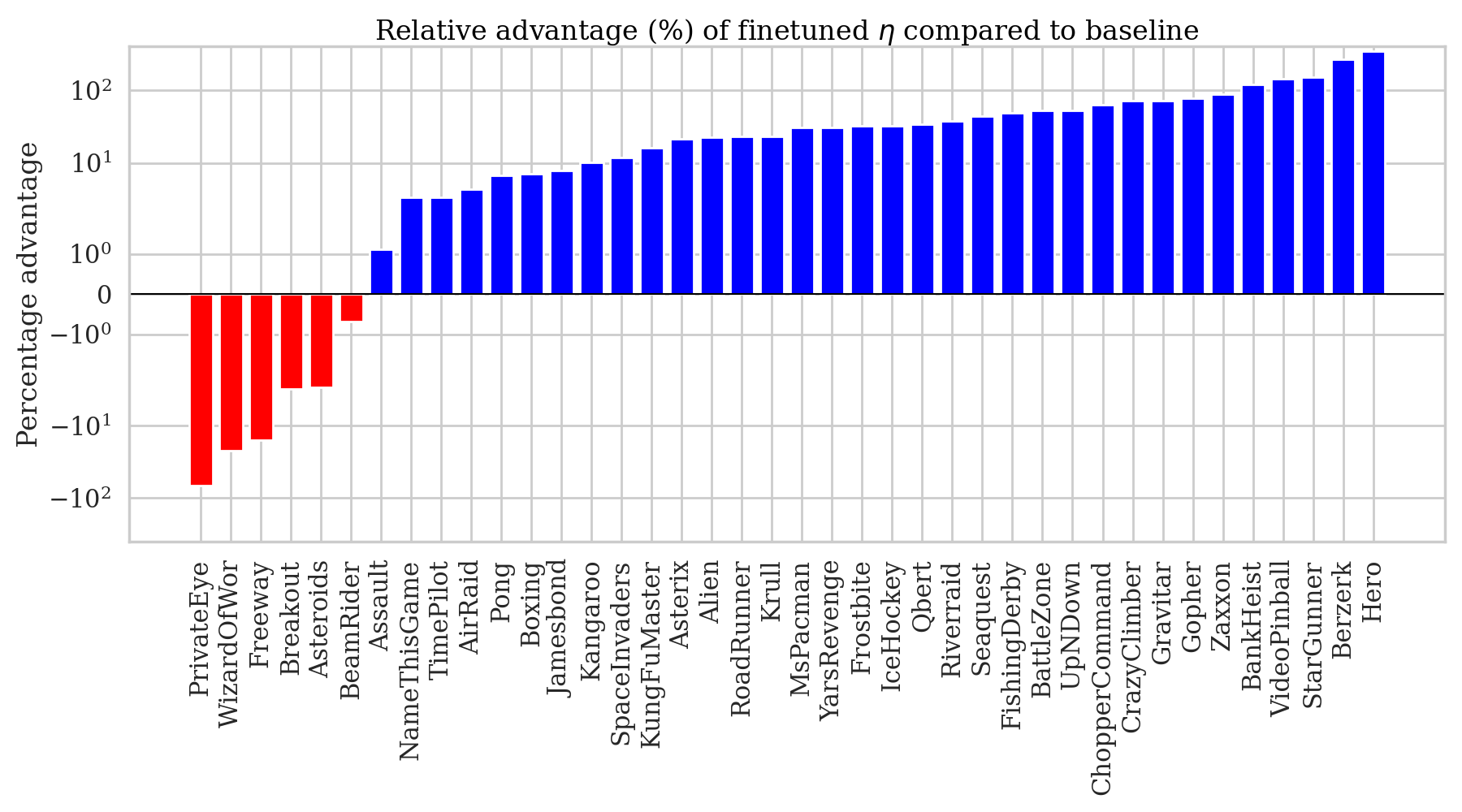}
  \caption{Relative advantage of the finetuned shape scale versus baseline ($\eta=0$) performance. Each environment, for each shape-scale parameter ${\eta\in\{0,0.5,1,2,3,5,10\}}$, is run with five seeds, and the best (in terms of mean score) non-zero $\eta$ value is chosen.  }
  \label{fig:env-comparison}
\end{minipage}
\end{figure*}
In this section, we derive some theoretical properties of BSRS. Specifically, we show that under appropriate scaling values (``shape-scales'') $\eta$, continual shaping in fact converges despite constant changes in the potential function during training. To prove this result, we employ usual techniques for contraction mappings. For this algorithm, we can directly calculate the asymptotic value functions and potential functions, which can be written in terms of the original un-shaped MDP's optimal value function. 
We then show that BSRS is \textit{not} equivalent to a constant shaping mechanism for any potential function; confirming that BSRS is a novel technique that cannot be replicated by a finetuned choice of potential function. Finally, we provide alternative interpretations of BSRS by re-considering the implied parameter updates under TD(0) and SARSA(0) learning rules.
% Motivate with \JA{Proof that using $V^*$ is better sample complexity: only need to learn SA-S=S(A-1) Q values rather than SA. Effective action-space reduced. Is reward shaping a regularization, as in https://arxiv.org/pdf/2007.02040?}
% Assuming that the prior policy for $K$ is the same as the prior policy for the original task ($\widetilde{Q}$), then in both standard and entropy-regularized RL we have:
% \begin{equation}
%     \widetilde{Q}(s,a) = \Phi(s) + K(s,a)
% \end{equation}
% For the base task ($Q$ above), we can take $\hat{Q}_\theta$, the previous estimate for the action-value function, as given by the function approximator. Since (by Thm 1), $K$ has the same optimal policy as $\widetilde{Q}$, one can instead learn $K$, which is equivalent to reward shaping on $\widetilde{r}$.
% Thus we consider $\hat{V}^*_\theta(s) = \max_{a} \hat{Q}_\theta(s,a)$ as the potential for reward shaping. 
% This alters the reward function to be of the form $\widetilde{r} + \gamma \mathbb{E} \hat{V}^*_\theta(s') - \hat{V}^*_\theta(s)$ where we constantly use the previous learning-step's output as the new base task.
\begin{tcolorbox}[colback=blue!5!white,colframe=blue!75!black]
\begin{theorem}\label{thm:bsrs}
Denote the optimal value functions for the unshaped MDP as $Q_0(s,a)$ and $V_0(s)$. Consider an algorithm which continuously reshapes the reward function at each step $n$ with the potential $\Phi^{(n)}(s)=\eta \max_a Q^{(n)}(s,a)$. Then, the operator
% At each step, reshape the reward by $\Phi^{(n)}(s)=\eta \max_a Q^{(n)}(s,a)$.
% Then, the corresponding operator
\begin{align*}
    \mathcal{T} Q(s,a) &= \left(r(s,a) + \gamma \E_{s'\sim{} p} \Phi^{(n)}(s') - \Phi^{(n)}(s) \right) \\
    &+ \gamma \E_{s'\sim{}p}\max_{a'} Q(s',a')
\end{align*}
remains a contraction mapping for values of ${\eta~\in (-1, (1-\gamma)/(1+\gamma))}$. The shaping potential and value function converge to: 
\begin{equation}
\Phi^{(\infty)}(s)=\frac{V_0(s)}{1+\eta},    
\end{equation}
% and the optimal $Q$-function, the fixed point of $\mathcal{T}$, is
\begin{equation}
    Q^{(\infty)}(s,a)=Q_0(s,a)-\frac{\eta}{1+\eta} V_0(s).
\end{equation}
% The function $Q^\infty(s,a)$ is the optimal action value function for an MDP with rewards $\left(r(s,a)  - \eta A_0(s,a) \right)/(1+\eta)$.
\end{theorem}
\end{tcolorbox}
For completeness, we provide the proof of this result below to give further insight on the mechanism at play and the Theorem's conclusions.

\begin{proof}
First we prove that the stated operator is indeed a contraction, before calculating the asymptotic potential and value functions. Consider step $n$ of training, where the new value function $Q^{(n)}$ is calculated from the previous iteration $Q^{(n-1)}$ via the Bellman backup equation, using PBRS with potential function $\Phi^{(n-1)}$:
\begin{align*}
   Q^{(n)}(s,a) &= \left(r(s,a) - \Phi^{(n-1)}(s) \right) + \\
   &\gamma \E_{s'\sim{} p(\cdot|s,a)}\left( \max_{a'} Q^{(n-1)}(s',a') + \Phi^{(n-1)}(s') \right).
\end{align*}
With the choice 
of potential specified by BSRS: ${\Phi^{(n-1)}(s)=\eta \max_a Q^{(n-1)}(s,a)}$, the corresponding Bellman operator above may be written as: 
\begin{align*}
    \mathcal{T} Q(s,a) &= \left(r(s,a) - \eta \max_a Q(s,a) \right) \\
    &+ \gamma(1+\eta) \E_{s'\sim{} p(\cdot|s,a)} \max_{a'} Q(s',a').
\end{align*}
\begin{figure*}[t]
\begin{minipage}[t]{1.0\textwidth}
  \includegraphics[width=\linewidth]{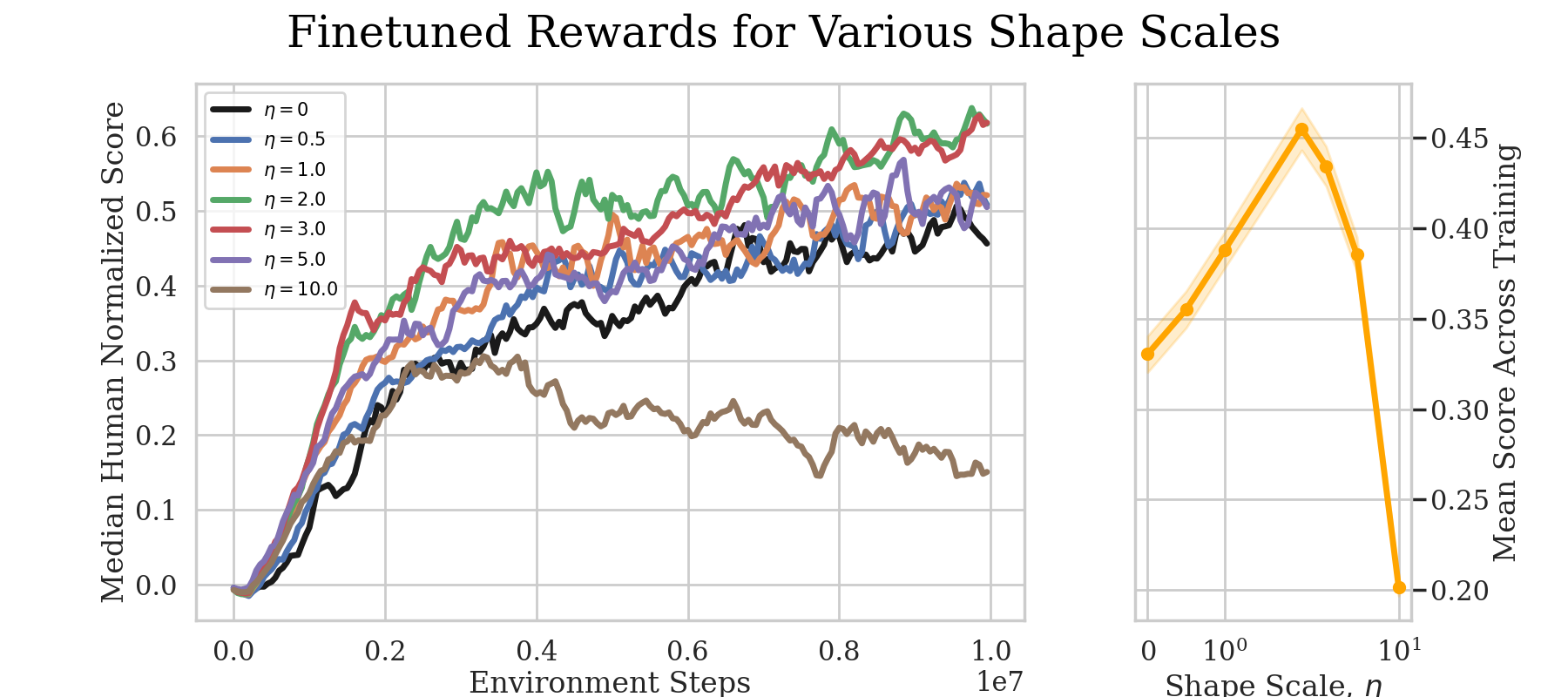}
  \caption{Learning curves for 10M steps in the Atari suite. We take the median human-normalized score over all $40$ environments (shown in Figure~\ref{fig:env-comparison}). In the right panel, a sensitivity plot is given, showing that an intermediate value of $\eta\approx2$ gives the best performance in aggregate (mean of the median human-normalized score over all environment steps).}
  \label{fig:rwd-curves}
\end{minipage}
\end{figure*}
To ensure that the Bellman operator with shaping (denoted $\mathcal{T}$ above) is indeed a contraction, we must verify that each application of the shaped Bellman operator reduces the distance between functions in the sup-norm. That is, we require ${|\mathcal{T}U - \mathcal{T}W|_\infty \leq \alpha |U-W|_\infty}$ to hold for some $\alpha \in [0,1)$, for all bounded functions $U$ and $W$. Proceeding directly with the calculation we find:

\begin{align*}
    &\big|\mathcal{T} U - \mathcal{T} W\big|_\infty \\
    &= \biggr| - \eta \max_a U(s,a) + \gamma (1+\eta) \Es \max_{a'} U(s',a') \notag \\
    &\quad + \eta \max_a W(s,a) - \gamma(1+\eta) \Es \max_{a'} W(s',a') \biggr|_\infty \notag \\
    &= \biggr| - \eta \left(\max_a U(s,a) - \max_a W(s,a) \right) \notag \\
    &\quad + \gamma (1+\eta)\Es \left( \max_{a'} U(s',a') - \max_{a'} W(s',a') \right) \biggr|_\infty \notag \\
    &\leq |\eta| \times \biggr|\max_a U(s,a) - \max_a W(s,a) \biggr|_\infty \notag \\
    &\quad + \gamma |1+\eta| \times \biggr|\Es \left( \max_{a'} U(s',a') - \max_{a'} W(s',a') \right) \biggr|_\infty \notag \\
    &\leq |\eta| \times \big|U - W\big|_\infty \notag \\
    &\quad + \gamma|1+\eta| \times \mathbb{E}_{s'\sim{} p} \biggr|\max_{a'} U(s',a') - \max_{a'} W(s',a') \biggr|_\infty \notag \\
    &\leq (|\eta| + \gamma |1 + \eta|) \big|U - W\big|_\infty,
\end{align*}

\noindent where the third line follows from the triangle inequality, and remaining lines follow the typical proof for the Bellman operator being a contraction.

To ensure $\mathcal{T}$ is a contraction, the constant factor must satisfy ${\alpha \doteq |\eta| + \gamma|1 + \eta| \in [0,1)}$, which can be solved for $\eta$: 
% The inequality is satisfied for a choice of $\eta$ in the range:
\begin{equation*}
    \eta~\in \left(-1, \frac{1-\gamma}{1+\gamma}\right).
\end{equation*}
\newline
% \JA{Assume $\eta>0$, then we have}
% \begin{align*}
%         \eta + \gamma + \eta \gamma &< 1\\
%         \eta(1+\gamma) &< 1-\gamma\\
%         \eta &< \frac{1-\gamma}{1+\gamma} \left(< \frac{1-\gamma}{2}\right)
% \end{align*}
% The lower bound gives
% \begin{align*}
%     \eta + \gamma + \eta \gamma &\geq 0\\
%     \eta(1 + \gamma) &\geq -\gamma\\
%     \eta &\geq -\frac{\gamma}{1 + \gamma}.
% \end{align*}
% \JA{Now assume $\eta<-1$, }

With the contractive nature of $\mathcal{T}$ verified, we invoke Banach's fixed point theorem stating that $\mathcal{T}$ has a unique fixed point. 
Denoting $\mathcal{T}^\infty Q(s,a)=Q^\infty(s,a)$ as the fixed point and $V^\infty(s)=\max_a Q^\infty(s,a)$ as the 
associated state value function, we can find the corresponding asymptotic
potential function ${\Phi^\infty(s)}$ by solving the following self-consistent equation,
\begin{equation}
   \Phi^\infty(s) = \eta \max_a \mathcal{T} Q^\infty(s,a)=\eta V^\infty(s)
\end{equation}
where $Q^\infty$ depends implicitly on $\Phi^\infty$. More explicitly, we can calculate the right-hand side of this self-consistent equation starting from $Q^\infty$ and then taking a maximum over action space:
\begin{align*}
    \mathcal{T}Q^\infty(s,a) &= \left(r(s,a) + \gamma \eta \Es V^{\infty}(s') - \eta V^{\infty}(s) \right) \\
    &\quad +\gamma \Es \max_{a'} Q^\infty(s',a') \\
    &= r(s,a) + \gamma(1+\eta) \Es V^\infty(s') - \eta V^{\infty}(s).
    % &= r(s,a) + \gamma \Es V_0(s') - \frac{\eta}{1+\eta}V_0(s) \\
    % &= Q_0(s,a) - \frac{\eta}{1+\eta}V_0(s).
\end{align*}
Taking the max over actions, we have:
\begin{align*}
    (1+\eta) V^\infty(s) &= \max_a \left\{ r(s,a) + \gamma (1+\eta) \Es V^{\infty}(s')  \right\}.
\end{align*}
Now we notice that a similar equation is solved by $V_0$ when $\eta=0$. Thus, if we assume the form $V_0=(1+\eta)V^\infty$, then the previous equation is satisfied:

\begin{align*}
    (1+\eta) V^\infty(s) &= \max_a \left\{ r(s,a) + \gamma \Es V_0(s')  \right\} \\
    (1+\eta) V^\infty(s) &= V_0(s) \\
    V^\infty(s) &= \frac{V_0(s)}{1+\eta},
\end{align*}
which is consistent with the aforementioned assumption. Now to solve for the fixed point, $Q^\infty$, we write out the backup equation from above, and insert the known expression for $V^\infty$:
\begin{align*}
    Q^\infty(s,a) &= r(s,a) + \gamma(1+\eta) \Es V^\infty(s') - \eta V^{\infty}(s) \\
    &= r(s,a) + \gamma \Es V_0(s') - \frac{\eta}{1+\eta} V_0(s) \\
    &= Q_0(s,a) - \frac{\eta}{1+\eta} V_0(s) 
\end{align*}

% \begin{align*}
%     \mathcal{T}Q^\infty(s,a) &= \left(r(s,a) + \gamma \eta \Es \Phi^{\infty}(s') - \eta \Phi^{\infty}(s) \right) \\
%     &\quad +\gamma \Es \max_{a'} Q^\infty(s',a') \\
%     &= r(s,a) + \gamma(1+\eta) \Es \Phi^\infty(s') - \eta\Phi^{\infty}(s) \\
%     &= r(s,a) + \gamma \Es V_0(s') - \frac{\eta}{1+\eta}V_0(s) \\
%     &= Q_0(s,a) - \frac{\eta}{1+\eta}V_0(s).
% \end{align*}

Calculating the associated asymptotic potential gives, as stated above,
\begin{align*}
\Phi^{\infty}(s) &= \eta \max_a Q_\infty(s,a) = \frac{\eta}{1+\eta}V_0(s).
\end{align*}
% showing that this value of $\Phi^\infty$ is indeed a fixed point of this shaping iteration.
\end{proof}

We note that due to Theorem 1 of \cite{Adamczyk_AAAI}, the previous result readily extends to the case of entropy-regularized RL.

\subsubsection{Initialization}
The previous proofs, in addition to the results of \cite{wiewiora2003potential}, suggest that the shaping method used is equivalent to a particular $Q$-table initialization, based on the shaping function:
\begin{equation*}
    Q^{(0)}(s,a) = \Phi^\infty(s) = \frac{V^*(s)}{1+\eta}.
\end{equation*}

However, the results in \cite{wiewiora2003potential} were only proven for \textit{static} potentials which do not change over the course of training. In fact, we find that BSRS is \textit{not} equivalent to shaping with \textit{any} static potential, as clarified in the following remark. (Proofs of this remark and other statements, alongside further experimental details, are given in the Appendix.)% of the full paper~\cite{arxiv-version}.)

\begin{remark}\label{rmk:static-pot}
    It can be shown that no equivalent static potential function exists with the same resulting updates as our BSRS potential.
\end{remark}
% The proof of this negative result is given in the Appendix.

\subsection{Interpretation as Un-Shaped Learning Problem}

The work of \cite{amit2020discount} showcased the connection between regularized algorithms and variable discount factors. In this section we take inspiration from their results and proof techniques to provide similar results for our setting.

\begin{proposition}[Scaled TD(0) Equivalence]
Let $\theta$ denote the parameters of the value function $V_\theta(s)$. The self-shaping algorithm, with shape-scale $\eta$ produces the same set of updates as an un-shaped problem setting, with a rescaled learning rate $\alpha \to \alpha(1+\eta)$ and rescaled reward function $r \to r/(1+\eta)$.
\end{proposition}

\begin{proof}
The proof follows the same techniques as the proof of Proposition 1 in \cite{amit2020discount}, without any extra regularization terms.
\end{proof}
This proposition suggests (at least for TD(0) value learning) that BSRS is equivalent to solving an appropriately rescaled MDP. Since we are operating in the un-regularized objective setting, a reward function being rescaled by a positive constant (which is enforced by the bounds on $\eta$) leaves the optimal policy invariant.

% This proposition allows us to give a basic intuition for a simplified training dynamics in sparse-reward situations: If the neural network $V_\theta$ is initialized such that $V_\theta(s)\approx0$ (as is typical in e.g. standard MLP initializations), then we can give a rough sketch for the mechanism in effect. Suppose that the agent takes roughly $k\gg1$ steps in the environment to reach a non-zero reward, and suppose that the update-to-data (UTD) (or ``replay'', RR) ratio is $\mathcal{O}(1)$. Then, after these $k$ steps, the value function will be trained to fit $V_{\theta_k}\approx\gamma^k V_{\theta_0}\to 0 $ as $k \to \infty$. Then, when the first non-zero reward is received, the value function is fit to $V_{\theta_{k+1}} = r$. 

% Now if $\eta>0$, the learning rate is enhanced while the reward scale is suppressed. Nevertheless, for bounded $\eta$ ($\eta<$

% \subsubsection{Interpretation as Regularized Objective}

However, our FA experiments do not employ TD(0) learning. Instead, we perform a SARSA-style update to the $Q$ network. Thus, although the previous proposition can provide some basic intuition, the true learning dynamics with self-shaping is more nuanced. Some of this nuance is captured in the following proposition, which extends the previous result from the TD setting to the SARSA setting.

\begin{proposition}[Regularized SARSA(0) Equivalence] \label{prop:sarsa}
Let $\theta$ denote the parameters of the value function $Q_\theta(s,a)$. Suppose the state-value function is calculated from $Q_\theta$ with a stop-gradient. The BSRS algorithm with shape-scale $\eta$ produces the same set of updates as an un-shaped problem setting, with a rescaled learning rate $\alpha \to \alpha(1+\eta)$, rescaled reward function $r \to r/(1+\eta)$, and regularized objective, with $\ell_2$-regularization on the advantage function.
\end{proposition}

\subsubsection{Function Approximation for $\Phi$}
Interestingly, we find that using the online, as opposed to target, network for calculating the potential drastically improves performance. This contradicts the expectation that the ``stability'' introduced by the target network should be beneficial for calculating a stable potential function. Rather, we find empirically that it is better to use the more rapidly updated online network to calculate $\Phi$. This is also in contradiction to Proposition~\ref{prop:sarsa}, which assumed the target network is used to calculate the potential. Nevertheless, the proof in the Appendix gives (in the penultimate line) the regularization term $ A_\theta(s,a) \nabla_\theta Q_\theta$ when an online network is used for calculating the potential function.

% \begin{figure*}[t]
% \begin{minipage}[t]{1.0\textwidth}
%   \includegraphics[width=\linewidth]{finetuned_advantage_atari10m_shape_scale.png}
%   \caption{Relative advantage of the finetuned shape scale versus baseline ($\eta=0$) performance. }
%   \label{fig:env-comparison}
% \end{minipage}
% \end{figure*}
\section{Experiments}
In our experiments, we consider the ``self-shaped'' version of the vanilla value-based algorithm DQN~\cite{mnih2015human, stable-baselines3}. We evaluate the performance of our self-shaped algorithm (BSRS) on a variety of environments in a tabular setting, the Atari suite, and on a continuous control task with TD3~\cite{fujimoto2018addressing}.

\subsection{Tabular Setting}
In the tabular setting, we can exactly solve the MDP and compare the shaped results to un-shaped learning curves and value functions for verification of the theory.
We find that the $Q$-values diverge near the correct boundaries of $\eta$, but our experiments suggest that the allowed range for $\eta$ can potentially be expanded (in the positive direction) beyond the values given above.

We plot the performance (in terms of number of steps until convergence) in Figure~\ref{fig:eta-steps}. The shaded region indicates the standard deviation across 5 random initializations of the $Q$-table. The inset plot shows the $7\times 7$ gridworld with sparse rewards used for this experiment (arrows indicate the optimal policy). Interestingly, we find that for much larger values of $\eta$ ($\approx 3$ times larger than the allowed maximum value) the performance continually improves, with an optimal value achieving roughly $20\%$ reduction in the number of required steps. The relative location of the optimal performance is robust to various environments, stochasticity, and discount factors. This phenomenon seems loosely analogous to the discussion on learning at the ``Edge of Stability'' in recent literature~\cite{cohen2021gradient, ahn2024learning}, but the exact connection is still unclear.

% \subsubsection{Exploration}
\begin{figure}[h]%{r}{0.42\textwidth}
  \begin{center}
    \includegraphics[width=0.45\textwidth]{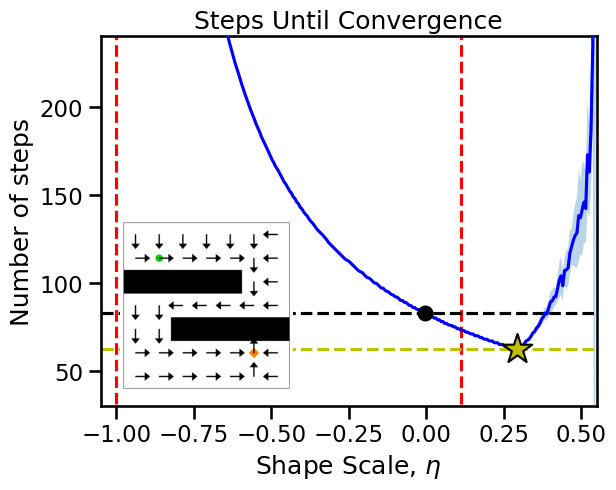}
  \end{center}
    \caption{Upon solving the self-shaped Bellman equation in the tabular setting (for environment at inset, $\gamma=0.8$), we find that increasingly large (even beyond the proven range, shown with red dashed lines) values of $\eta$ allows for improved performance. We define ``convergence'' in this case as the point at which error between iterates in the sup-norm falls below $10^{-6}$. On the $y$-axis, ``steps'' refers to the number of applications of the self-shaped Bellman operator before convergence. Each point on the curve corresponds to an average over $20$ random initializations.}
    \label{fig:eta-steps}
\end{figure}

\subsection{Continuous Setting}
For more complex environments, we use the Arcade Learning Environment suite~\cite{ALE}. The results of these experiments are shown in Figures~\ref{fig:env-comparison} and~\ref{fig:rwd-curves}. Across many environments, BSRS leads to an improvement over the baseline DQN performance ($\eta=0$). In $21/40$ environments we find an improvement of more than $10\%$ and $5/40$ environments show an improvement of over $100\%$. In only $6/40$ environments does self-shaping have a negative impact. If $\eta$ is to be tuned over, one can simply choose $\eta=0$ for these environments. In Figure~\ref{fig:rwd-curves} we find an improvement in the aggregate reward curves of $45\%\to60\%$ human normalized score. We find that the value of $\eta$ can only be increased up until a point ($\eta\approx2$) until the performance deteriorates. Although our Theorem~\ref{thm:bsrs} suggests that $\eta$ can be negative, we found that the performance in this regime is even worse than that observed for $\eta=10$. Overall, we find a substantial speedup at small times ($\sim{}1.5$M steps) and a lasting improvement in rewards at long times for multiple $\eta$ values.

To test BSRS in the continuous action setting, we use Pendulum-v1 by extending an implementation of TD3~\cite{stable-baselines3}. Since TD3 directly maintains an estimate of both the policy and the value function, the latter is used to derive the dynamic shaping potential, $\Phi(s)$. In principle, one must maximize over a continuous set of actions at each timestep to calculate $\Phi$, which becomes intractable, so we instead use the action sampled by the actor. The addition of BSRS to TD3 shows promising improvements over the baseline ($\eta=0$) with a notably more robust performance over a range of large $\eta$ values: ~Fig.~\ref{fig:td3-pend} and inset. We also show the result for $\eta<0$ in the inset, which leads to worse performance in all the FA settings considered. Further experiments are shown in the Appendix.

\begin{figure}[t]
    \centering
    \includegraphics[width=0.45\textwidth]{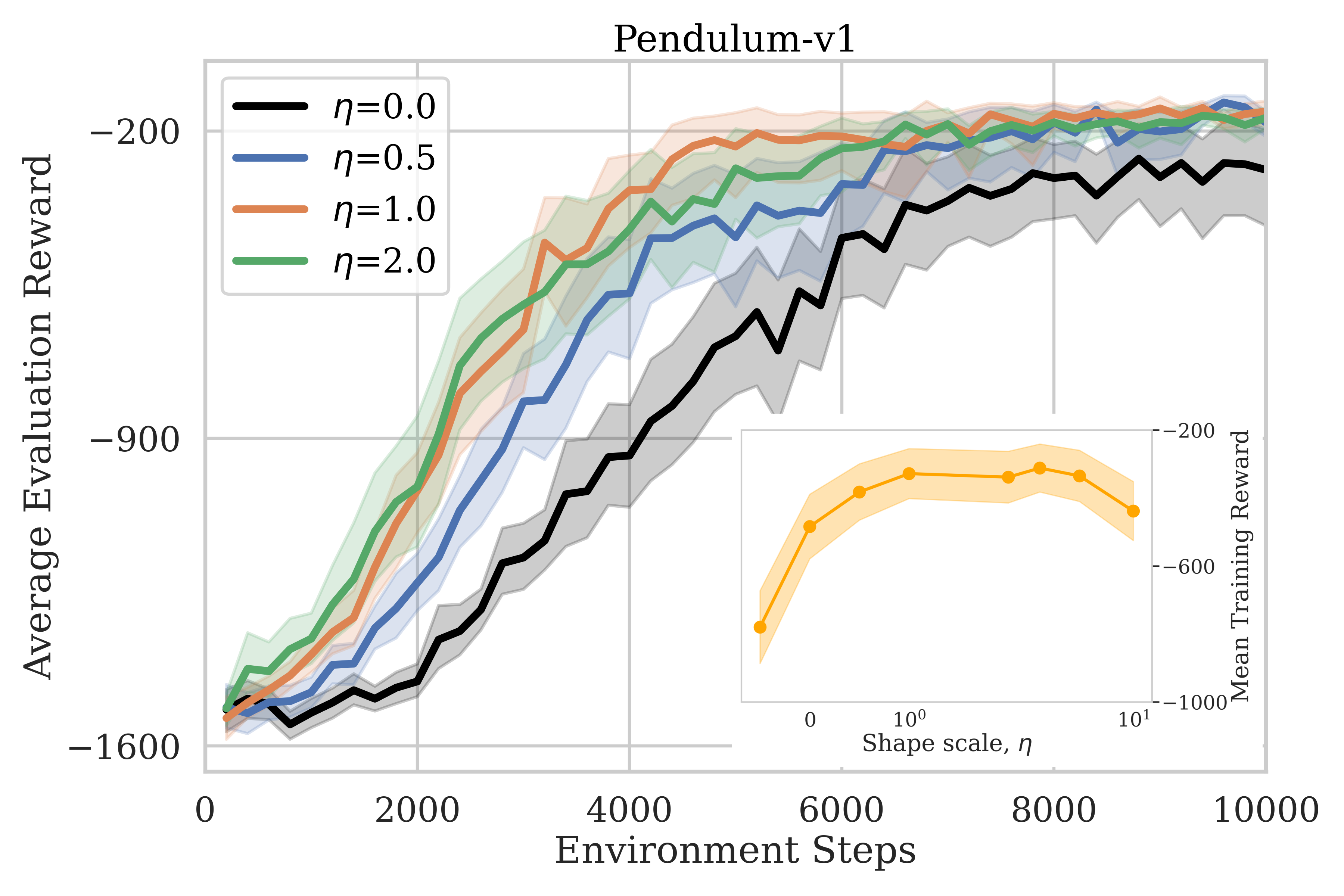}
    \caption{TD3 with BSRS is tested on a continuous control task, Pendulum-v1. {Inset: Robustness to varying~$\eta$. Results for each $\eta$ value were averaged over $20$ runs (standard error indicated by shaded region).}} %In this task, the performance is more robust to large values of $\eta$, especially compared to the Atari domain shown in Fig.~\ref{fig:rwd-curves}.}
    \label{fig:td3-pend}
\end{figure}

\section{Discussion}

In this work, we provided a theoretically-grounded choice for a dynamic potential function. This work fills the gap in existing literature by providing a universally applicable form for the potential function in any environment. Notably, rather than attempting to tune over the \textit{function class} $\Phi:\s\to\R$, we instead suggest to tune over the significantly simpler \textit{scalar class} $\eta\in\R$. This idea simplifies the problem of choosing a potential function from a high-dimensional search problem to a single hyperparameter optimization. 

Future work can naturally extend the results presented. For instance, one may study techniques to learn the optimal value of $\eta$ over time, perhaps analogous to the method of learning the $\alpha$ value in \cite{haarnoja2018soft}. Further theoretical work can be pursued for understanding the convergence properties of BSRS. For instance, it appears numerically that values of $\eta$ beyond the proven bounds can be used. Also, it is straightforward to see in the proof of Theorem~\ref{thm:bsrs} that all instances of $\eta$ may be replaced with the functional $\eta(s)$, giving further control over the self-shaping mechanism. Future work may study such state-dependent shape scales, e.g. dependent on visitation frequencies or the loss experienced in such states (cf.~\cite{wang2024efficient}), which can further connect to the problem of exploration.

Overall, our work provides a practically relevant implementation of PBRS which provides an advantage in training for both tabular and deep RL.%cases, such as the Atari suite. 

\section{Acknowledgements}
JA would like to acknowledge the use of the supercomputing facilities managed by the Research Computing Department at UMass Boston; the Unity high-performance computing cluster; and funding support from the Alliance Innovation Lab – Silicon Valley. RVK and JA would like to acknowledge funding support from the NSF through Award No. DMS-1854350 and PHY-2425180.
ST was supported in part by NSF Award (2246221), PAZY grant (195-2020), and WCoE, {Texas~Tech~U.} This work is supported by the National Science Foundation under Cooperative Agreement PHY-2019786 (The NSF AI Institute for Artificial Intelligence and Fundamental Interactions, http://iaifi.org/).

% This work showed empirically that random dynamic potentials can be quite harmful to the training behavior, whereas potentials derived from ``negative bias'' information can be somewhat beneficial for training speed. \JA{We see some correlation in our results, with slight negative eta improving, and maybe random @init can worsen performance. This suggests that random is bad, so maybe should wait before turning on eta...}

% \JA{

% \begin{itemize}
%     \item current active Q, target Q, and new timescale for even slower
%     \item plot variance in gamma V(s') - V(s)
%     \item 
% \end{itemize}

% }

\bibliography{aaai25}
% \clearpage
\iftrue
\onecolumn
\appendix
\section{Proofs}
% State-dependent shape scales are also valid within the same theoretical framework. The constraint for ensuring the contraction operator is now
% \begin{equation}
%     |\eta(s)| + \gamma|\eta(s)+1| < 1
% \end{equation}
\subsection{Proof of Remark~\ref{rmk:static-pot}}

\begin{proof}
We will proceed by contradiction. Suppose there exists a static function $\Psi(s)$ such that all iterations of BSRS are in agreement with the fixed potential $\Psi(s)$. Then,
at steps $n-1$ and $n$ of training (two applications of the Bellman operator), the following equations must agree for some choice of $\Psi(s)$:

\begin{align}
    &r(s,a) + \gamma \max_{a'} \left\{ r(s',a') + \gamma \left(V^{(n-1)}(s'') + \Psi(s'') \right) - \Psi(s') \right\} +\gamma \Psi(s') - \Psi(s)  \\
    &r(s,a) + \gamma \underbrace{\max_{a'} \left\{ r(s',a') + \gamma \left((1+\eta) V^{(n-1)}(s'') \right) - \eta V^{(n-1)}(s') \right\}}_{V^{(n)}(s')} +\gamma \eta V^{(n)}(s') - \eta V^{(n)}(s) 
\end{align}
This is equivalent to the following equations agreeing (simplifying shared terms and canceling subsequent potentials in the static case):

\begin{align}
    & \gamma \max_{a'} \left\{ r(s',a') + \gamma \left(V^{(n-1)}(s'') + \Psi(s'') \right) \right\} - \Psi(s)  \\
    &\gamma \max_{a'} \left\{ r(s',a') + \gamma \left((1+\eta) V^{(n-1)}(s'') \right) - \eta V^{(n-1)}(s') \right\} +\gamma \eta V^{(n)}(s') - \eta V^{(n)}(s)
\end{align}
Since the reward function is arbitrary and the equation must hold for all states, the inner terms must agree:

\begin{equation}
    V^{(n-1)}(s'') + \Psi(s'') = (1+\eta) V^{(n-1)}(s''),
\end{equation}
and so we must have the relation

\begin{equation}
     \Psi(s) = \eta V^{(n-1)}(s),
\end{equation}
which is inconsistent as a \textit{static} potential function, and also is in disagreement with the relation implied by the outer terms:
\begin{equation}
    \Psi(s)=-\gamma \eta V^{(n-1)}(s')+\gamma \eta V^{(n)}(s') - \eta V^{(n)}(s).
\end{equation}
Explicitly, we now see the disconnect is caused by the difference between subsequent value functions:
\begin{equation}
    \gamma \eta \left(V^{(n)}(s') - V^{(n-1)}(s') \right) \neq 0 ,
\end{equation}
which controls the error in the constant-potential function.
\end{proof}

\subsection{Proof of Proposition~\ref{prop:sarsa}}
\begin{proof}
We write out the difference between successive parameters, using the same proof technique as in~\cite{amit2020discount}:
\begin{align*}
    \theta'-\theta &= \alpha \nabla_\theta Q_\theta(s,a) \left[r+\gamma \Phi(s') - \Phi(s) + \gamma V(s') - Q_\theta(s,a) \right] \\
    &= \alpha \nabla_\theta Q_\theta(s,a) \left[r+\gamma \eta V(s') - \eta V(s) + \gamma V(s') - Q_\theta(s,a) \right] \\
    &= \alpha(1+\eta) \nabla_\theta Q_\theta(s,a) \left[\frac{r}{1+\eta}+\gamma V(s') - \frac{\eta}{1+\eta} V(s) - \frac{Q_\theta(s,a)}{1+\eta} \right] \\
    &= \widetilde{\alpha} \biggr( \nabla_\theta Q_\theta(s,a) \left[\widetilde{r}+\gamma V(s') -Q_\theta(s,a) \right] + \nabla_\theta Q_\theta(s,a) \left[- \frac{\eta}{1+\eta} V(s) + \frac{\eta}{1+\eta} Q_\theta(s,a) \right] \biggr) \\
    &= \widetilde{\alpha} \biggr( \nabla_\theta Q_\theta(s,a) \left[\widetilde{r}+\gamma V(s') -Q_\theta(s,a) \right] + \frac{\eta}{1+\eta}A_\theta(s,a) \nabla_\theta Q_\theta(s,a)  \biggr)\\
    &= \widetilde{\alpha} \biggr( \nabla_\theta Q_\theta(s,a)     \left[\widetilde{r}+\gamma V(s') -Q_\theta(s,a) \right] + \lambda \nabla_\theta A^2_\theta(s,a) \biggr)
\end{align*}
We denote the regularization coefficient as $\lambda \doteq \frac{1}{2} \frac{\eta}{1+\eta}$.
The last line follows if we assume that the gradient has no effect on the state value function (i.e. $V(s)$ is calculated via the target network or a stop-gradient operation is used):
\begin{equation}
    \nabla_\theta A_\theta(s,a) =\nabla_\theta \left(Q_\theta(s,a) - V(s)\right) = \nabla_\theta Q_\theta(s,a).
\end{equation}

\end{proof}
\section{Experiment Details}
We extend DQN and TD3 from Stable-baselines3~\cite{stable-baselines3} to include BSRS, by changing the definition of reward values before the target values are calculated.
\subsection{Hyperparameters}
For the Atari environments, we use all the hyperparameters from~\cite{mnih2015human}. Each run is trained for $10$M steps. All runs are averaged over $5$ random initializations in every environment. This yields a compute cost of roughly $10^7$ steps/run $\times 5$ runs/env. $\times 40$ envs. $\times 7$ shape-scales $\times 1/100$ sec./step $\approx 1600$ GPU-days. Training was performed across clusters with various resources, including A100s, V100s, and RTX 30 \& 40 series GPUs.
We use the standard set of Atari wrappers. Specifically, this include no-op reset, frame skipping (4 frames), max-pooling of two most recent observations, termination signal when a life is lost, resize to $84\times 84$, grayscale observation, clipped reward to $\{-1, 0, 1\}$, frame stacking (4 frames) and image transposing.
\begin{table}[h!]
\centering
\begin{tabular}{|l|c|}
\hline
\textbf{Parameter} & \textbf{Value} \\ \hline
Batch Size & $32$ \\ \hline
Buffer Size & $100,000$ \\ \hline
Discount Factor ($\gamma$) & $0.99$ \\ \hline
Gradient Steps & $1$ \\ \hline
Learning Rate & $0.0001$ \\ \hline
Target Update Interval & $1,000$ \\ \hline
Train Frequency & $4$ \\ \hline
Learning Starts & $50,000$ \\ \hline
Exploration Fraction & $0.1$ \\ \hline
Exploration Initial $\epsilon$ & $1.0$ \\ \hline
Exploration Final $\epsilon$ & $0.02$ \\ \hline
Discount Factor ($\gamma$) & $0.99$ \\ \hline
Total Timesteps & $10,000,000$ \\ \hline
\end{tabular}
\caption{Shared hyperparameter settings used in our ALE experiments for all values of $\eta \in \{0.0, 0.5, 1.0, 2.0, 3.0, 5.0, 10.0\}$.}
\label{tab:hyperparams}
\end{table}

For the Pendulum environment, we use TD3's default hyperparameters with a few modifications, based on hyperparameters published by~\cite{stable-baselines3}:
\begin{table}[h!]
\centering
\begin{tabular}{|l|c|}
\hline
\textbf{Parameter} & \textbf{Value} \\ \hline
Batch Size & $256$ \\ \hline
Buffer Size & $200,000$ \\ \hline
Discount Factor ($\gamma$) & $0.98$ \\ \hline
Gradient Steps & $-1$ \\ \hline
Learning Rate & $0.001$ \\ \hline
Learning Starts & $0$ \\ \hline
Train Frequency & $1$ \\ \hline
Policy Delay & $2$ \\ \hline
Target Policy Noise & $0.2$ \\ \hline
Target Noise Clip & $0.5$ \\ \hline
\end{tabular}
\caption{TD3 hyperparameters for Pendulum-v1 for all values of $\eta \in \{-0.5, 0.0, 0.5, 1.0, 2.0, 3.0, 5.0, 10.0\}$.}
\label{tab:additional_hyperparams}
\end{table}

Each run is trained for $20$k steps. All runs are averaged over $20$ random initializations in every environment. This yields a compute cost of roughly $10^4$ steps/run $\times 20$ runs/env. $\times 8$ shape-scales $\times 1/100$ sec./step $\approx 6$ GPU-hours. Training was performed locally with a single RTX 3080.
\clearpage
\section{Additional Experiments}
\begin{figure}[h]
    \centering
    \includegraphics[width=0.5\textwidth]{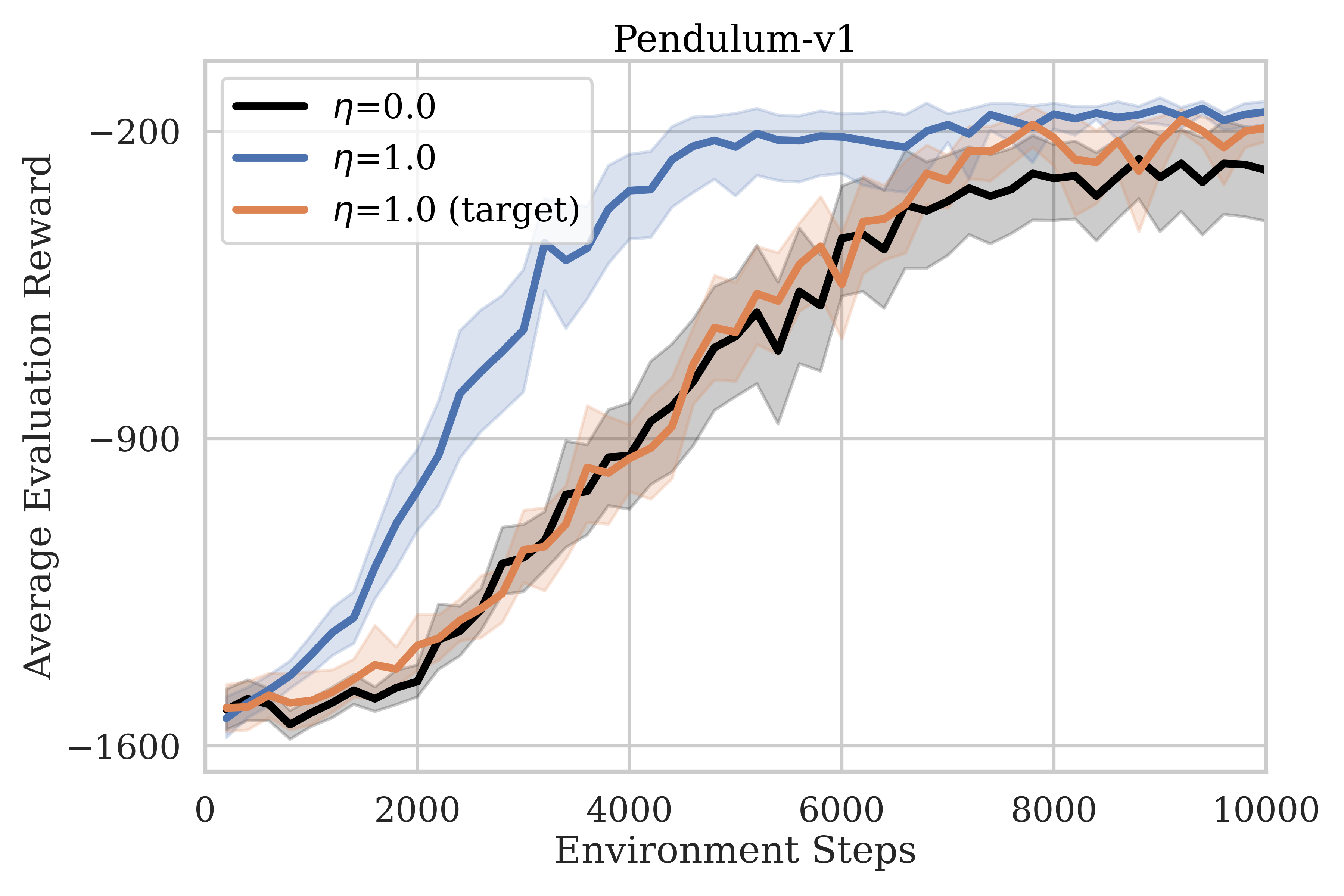}
    \caption{We also test using the target network to calculate $\Phi(s)$, rather than the online critic network. We find that despite the initial intuition that a target network may help stabilize the otherwise fast-changing reward structure, the use of the online network to derive $\Phi(s)$ is crucial to see a (positive) performance difference, at least in the settings considered here.}
    \label{fig:targ}
\end{figure}

\begin{figure}[h]
\centering
\begin{subfigure}{.5\textwidth}
  \centering
  \includegraphics[width=.9\linewidth]{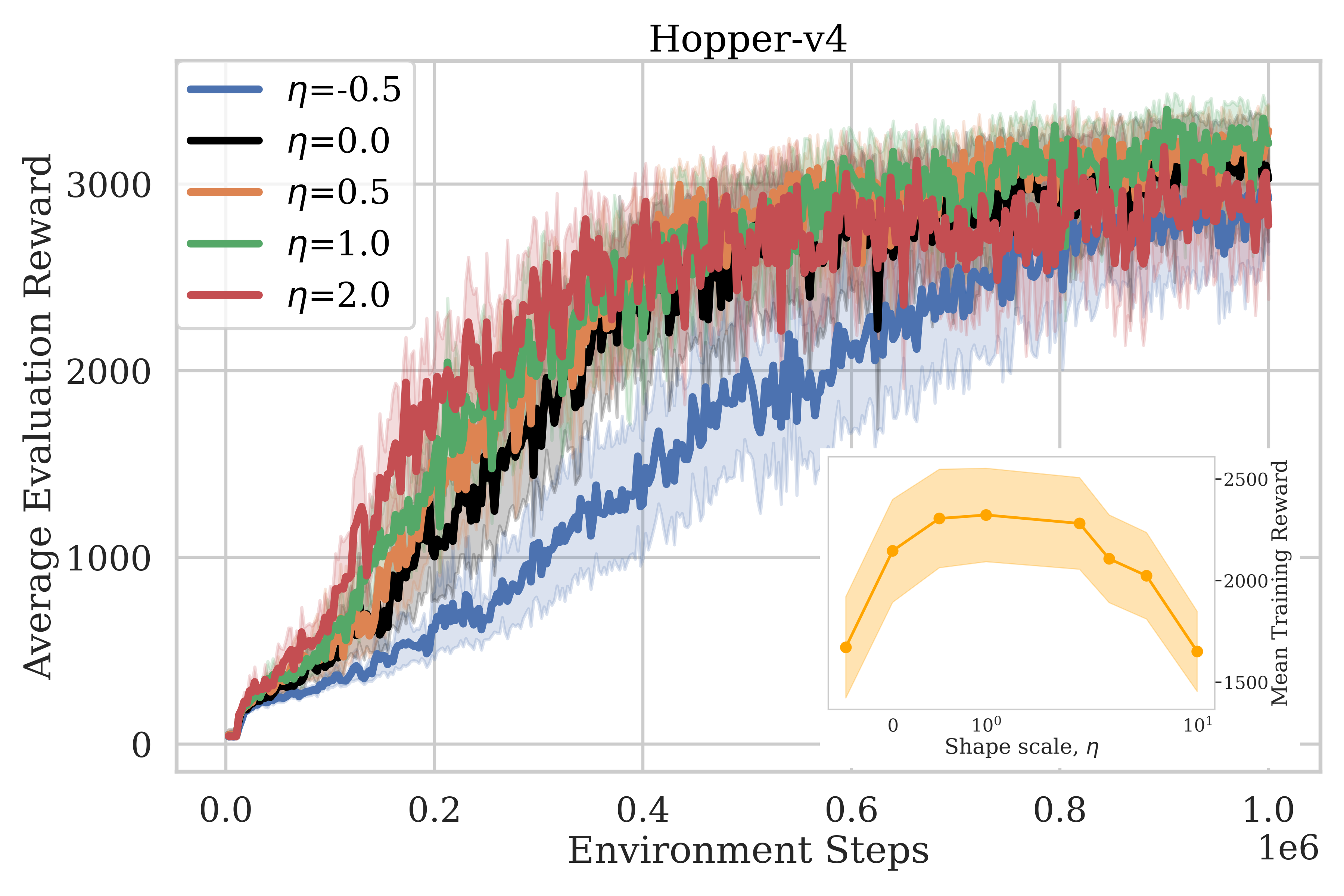}
  \caption{TD3 with BSRS on Hopper-v4.}
  \label{fig:sub1}
\end{subfigure}%
\begin{subfigure}{.5\textwidth}
  \centering
  \includegraphics[width=.9\linewidth]{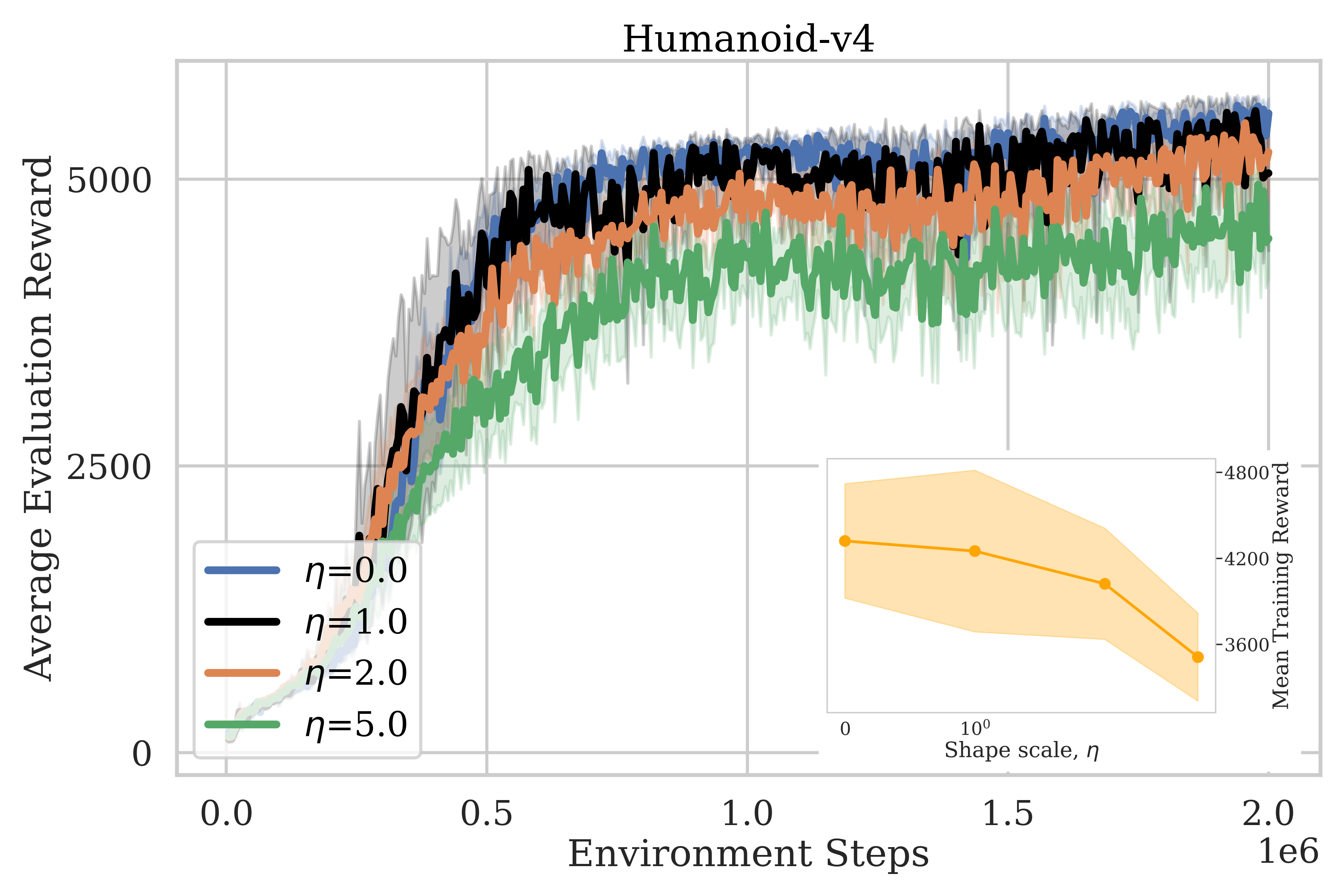}
  \caption{TD3 with BSRS on Humanoid-v4.}
  \label{fig:sub2}
\end{subfigure}
\caption{We also test BSRS on TD3 for more complex continuous action tasks. The performance is slightly better (at least for Hopper-v4) for an optimized value of $\eta$, but potentially within the margin of statistical significance. The hyperparameters are the same as those of Pendulum (shown in Table~\ref{tab:additional_hyperparams}) except for the learning rate of $3 \times 10^{-4}$ and an updated ``learning starts'' value of $10,000$ steps. For Humanoid-v4, we did not find a statistically significant improvement for the values of $\eta$ tested; however, we believe this may be improved in future work. As discussed in the main text, ideas for further improvement include: further hyperparameter tuning, state-dependent scale parameter $\eta(s)$, and dynamic schedules for the scaling parameter. Training was distributed locally and on clusters (with RTX 20, 30, and 40 series), totaling roughly 400 GPU-hours of compute.}
\label{fig:test}
\end{figure}

\fi
\end{document}